\documentclass[letterpaper, 10pt, conference]{ieeeconf}  

\IEEEoverridecommandlockouts                              

\usepackage[pdftex]{graphicx}
\usepackage{cite}
\usepackage{verbatim}		
\usepackage{amsmath}
\usepackage{amssymb}

\usepackage{amsthm}
\usepackage{mathtools}	
\usepackage{grffile}	
\usepackage[tight,footnotesize]{subfigure}
\usepackage{microtype} 
\usepackage{color}
\usepackage{url}
\usepackage[ruled,vlined,linesnumbered]{algorithm2e}
\usepackage{bm} 
\usepackage{comment}
\usepackage{arydshln}
\let\labelindent\relax
\usepackage{enumitem}

\usepackage{adjustbox}
\usepackage{tikz}
\usetikzlibrary{shadings, patterns, angles, quotes, arrows.meta, shapes, decorations.pathmorphing, decorations.shapes, decorations.text}
\usetikzlibrary{calc,intersections,arrows.meta}
\usepackage{pgfplots}

\usepackage{hyperref}
\hypersetup{
	colorlinks=true,
	linkcolor=blue,
	citecolor=red,
}

	\tikzset{
	pil/.style={
		->,
		thick,
		shorten <=2pt,
		shorten >=2pt,}
}

\theoremstyle{remark}
\newtheorem{remarkx}{Remark}
\newenvironment{remark}
{\pushQED{\qed}\remarkx}
{\popQED\endremarkx}

\theoremstyle{definition}
\newtheorem{defn}{Definition}

\newtheorem{problem}{Problem}
\newtheorem*{problem*}{Problem}

\theoremstyle{plain}
\newtheorem{theorem}{Theorem}
\newtheorem{lemma}{Lemma}

\newtheorem{prop}{Proposition}

\newcommand{\cmmnt}[1]{}

\newcommand{\scalemath}[2]{\scalebox{#1}{\mbox{\ensuremath{\displaystyle #2}}}}

\title{\LARGE \bf
Resilient source seeking with robot swarms 
}

\author{Antonio Acuaviva, Jesus Bautista, Weijia Yao, Juan Jimenez, Hector Garcia de Marina 
	\thanks{A. Acuaviva and J. Jimenez are with the Department of Computer Architecture and Automation, Universidad Complutense de Madrid, Spain. Jesus Bautista and H. Garcia de Marina are with the Department of Computer Engineering, Automation, and Robotics, and the Research Centre for Information and Communication Technologies (CITIC-UGR), University of Granada, Granada, Spain. Weijia Yao is with the School of Robotics, Hunan University. Corresponding author e-mail: {\tt\small hgdemarina@ugr.es}. The work of H.G. de Marina is supported by the ERC Starting Grant \emph{iSwarm} 101076091 and the grant Ramon y Cajal RYC2020-030090-I from the Spanish Ministry of Science.}%
}

\begin{document}

\maketitle
\thispagestyle{empty}
\pagestyle{empty}

\begin{abstract}
We present a solution for locating the source, or maximum, of an unknown scalar field using a swarm of mobile robots. Unlike relying on the traditional gradient information, the swarm determines an ascending direction to approach the source with arbitrary precision. The ascending direction is calculated from field strength measurements at the robot locations and their relative positions concerning the swarm centroid. Rather than focusing on individual robots, we focus the analysis on the density of robots per unit area to guarantee a more resilient swarm, i.e., the functionality remains even if individuals go missing or are misplaced during the mission. We reinforce the algorithm's robustness by providing sufficient conditions for the swarm shape so that the ascending direction is almost parallel to the gradient. The swarm can respond to an unexpected environment by morphing its shape and exploiting the existence of multiple ascending directions. Finally, we validate our approach numerically with hundreds of robots. The fact that a large number of robots with a generic formation always calculate an ascending direction compensates for the potential loss of individuals.
\end{abstract}


\section{Introduction}
\subsection{Aim and objectives}
The source seeking of a scalar field is regarded as one of the fundamental problems in swarm robotics due to its enormous potential for contemporary challenges \cite{yang2018grand,brambilla2013swarm}. Indeed, the ability to detect and surround sources of chemicals, pollution, and radio signals effectively with robot swarms will enable persistent missions in vast areas for environmental monitoring, search \& rescue, and precision agriculture operations \cite{ogren2004cooperative, kumar2004robot, mcguire2019minimal, li2006moth,twigg2012rss}. 

Robot swarms are among the most promising multi-robot systems because of their high-resilience potential, i.e., to preserve functionality against unexpected adverse conditions and unknown, possibly unmodeled, disturbances. In this regard, we aim to develop a source-seeking algorithm for robot swarms with guarantees that the robots reach the source even if (disposable) individuals go missing during the mission. The algorithm guarantees an ascending direction rather than the gradient in order to guide the centroid of the robot swarm to the source, i.e., all the individual robots track such a direction. Its calculation does not require a specific formation; in fact, by morphing its shape, the swarm adapts to the environment, and missing robots are not an issue since the \emph{resultant} formation is still valid. Calculating the ascending direction requires the robots to measure the strength of the scalar field at their current positions and then share this information with a \emph{computing-unit} robot. The computing unit does not have any specific requirements, and, in fact, redundant computing units can be employed to enhance the robustness of the swarm. Additionally, the robots must be aware of their centroid. The main objective of this paper is to study and determine the properties of the proposed ascending direction and its application to robot swarms.

\subsection{Literature review}
The following aims to assess the pros and cons of different algorithms, considering that robot systems are tailored to specific missions and requirements, e.g., robot dynamics, traveled distance, or communication topologies. 

The widely-used approaches focus on estimating the gradient and the Hessian of an unknown scalar field. In works like \cite{rosero2014cooperative, barogh2017cooperative}, the gradient estimation occurs in a distributed manner, where each robot calculates the different gradients at their positions. A common direction for the team is achieved through a distance-based formation controller that maintains the cohesion of the robot swarm. It is worth noting that this method fails when a subset of neighboring robots adopt a \emph{degenerate} shape, like a line in a 2D plane. Alternatively, in studies like \cite{brinon2015distributed, brinon2019multirobot}, the gradient and Hessian are estimated at the centroid of a circular formation consisting of unicycles. However, it is relatively rigid due to its mandatory circular formation. Such a formation of unicycles is also present in \cite{fabbiano2016distributed}, although robots are not required to measure relative positions but rather relative headings. In contrast to our algorithm, which calculates an ascending direction instead of the gradient, our robot team enjoys greater flexibility concerning formation shapes while still offering formal guarantees for reaching the source. 


Extremum seeking is another extensively employed technique, as discussed in \cite{li2020cooperative, cochran2009nonholonomic}. In these studies, robots, which could potentially be just one, with nonholonomic dynamics execute periodic movements to facilitate a gradient estimation. However, when multiple robots are involved, there is a necessity for the exchange of estimated parameters. Furthermore, the eventual trajectories of the robots are often characterized by long and intricate paths with abrupt turns. In contrast, our algorithm generates smoother trajectories, albeit requiring more than two robots.

All the algorithms above require communication among robots sharing the strength of the scalar field. Nonetheless, the authors in \cite{al2021distributed} offer an elegant solution involving the principal component analysis of the field that requires no sharing of the field strength. However, the centroid of the swarm is still necessary. It is worth noting that the initial positions of the robots determine the (time-varying) formation during the mission, and it is not under control. In comparison, although our algorithm requires the communication of the field strength among the robots, it seems to fit well with formation control laws resulting in more predictable trajectories.

\subsection{Organization}
The article is organized as follows. In Section \ref{sec: pre}, we introduce notations and assumptions on the considered scalar field and formally state our source-seeking problem. We continue in Section \ref{sec: tools} on how to calculate an ascending direction and analyze its sensitivity concerning the swarm and the observability of the gradient. In Section \ref{sec: sianalysis}, we analyze the usage of the ascending direction with robots modeled as single integrators. We validate our theoretical findings in Section \ref{sec: sce} involving hundreds of robots. Finally, we end this article in Section \ref{sec: con} with conclusions and upcoming results.

\section{Preliminaries and problem formulation}
\label{sec: pre}

Consider a team of $N \in \mathbb{N}$, with $N > m$ robots, where the position of the $i$-th robot in the Cartesian space is represented by $p_i \in \mathbb{R}^m$, where $m \in \{2,3\}$, $i \in \{1,\dots,N\}$. We define $p\in\mathbb{R}^{mN}$ as the stacked vector of the positions of all robots in the team. We will focus on the single integrator dynamics, i.e., the robot $i$ is modeled by
\begin{equation}
\label{eq: sid}
\dot p_i = u_i,
\end{equation}
where $u_i\in\mathbb{R}^m$ is the guiding velocity as control action. We define the centroid of the swarm as $p_c := \frac{1}{N}\sum_{i=1}^{N}p_i$; therefore, we can write $p_i = p_c + x_i$, where $x_i\in\mathbb{R}^m$ for all $i\in\{1,\dots,N\}$ describes how robots are spread around their centroid as it is shown in Figure \ref{fig: deployment}.

\begin{defn}
The \emph{deployment} \label{deployment} of a swarm is defined as the stacked vector $x := \begin{bmatrix}x_1^T, \dots, x_N^T\end{bmatrix}^T \in \mathbb{R}^{mN}$. We say a deployment is \emph{non-degenerate} if the elements in $x$ span $\mathbb{R}^m$.
\end{defn}

\begin{figure}
\centering
\includegraphics[trim={3cm 8.5cm 0 9.5cm}, clip, width=0.75\columnwidth]{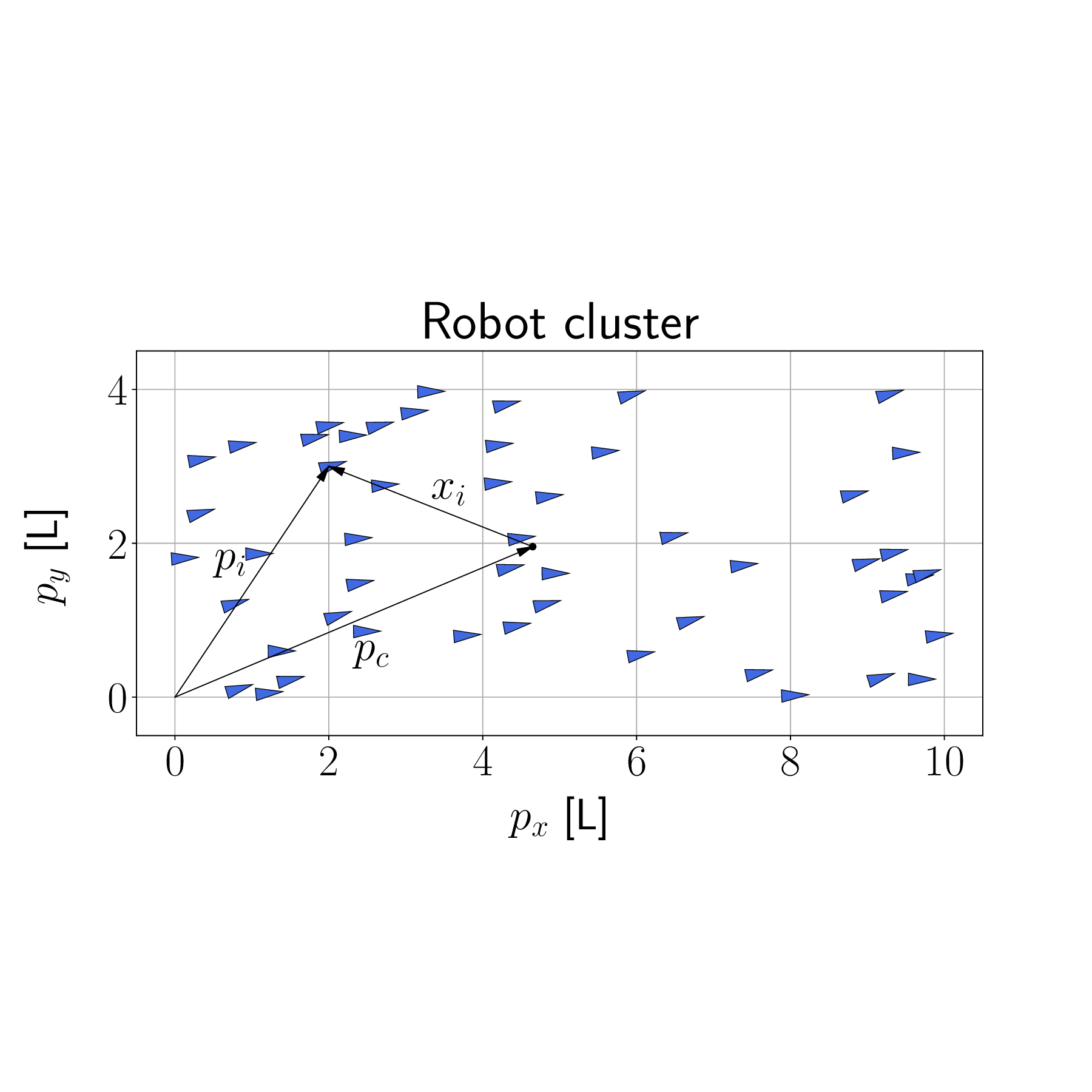}
\caption{Deployment $x$ of a robot swarm with centroid at $p_c$.}
\label{fig: deployment}
\end{figure}
Therefore, a necessary condition for a non-degenerate deployment is $N > m$. The strength of a signal throughout the space can be described by a scalar field.

\begin{defn} \label{signal} A \emph{signal} is a scalar field $\sigma: \mathbb{R}^m \to \mathbb{R}^+$ which is $C^2$ and all its partial derivatives up to second order are bounded globally. Also, $\sigma$ has only one maximum at $p_\sigma \in \mathbb{R}^m$, i.e., the source of the signal, and its gradient at $a\in\mathbb{R}^m$ satisfies $\nabla\sigma(a) \neq 0 \iff a \neq p_\sigma$ and $\lim_{|a|\to\infty}\sigma(a) = 0$.
\end{defn}

Our definition of signal fits plenty of models in our physical world, e.g., Gaussian distributions and the power law $\frac{1}{r^{2}}$, smoothly modified at the origin to avoid singularities.


In this paper, the gradient is defined as a column vector, i.e., $\nabla\sigma(\cdot) \in\mathbb{R}^m$, and according to our definition of a signal:
\begin{equation} \label{eq_grad_hess_bound}
\|\nabla\sigma(a)\| \leq K \quad \text{and} \quad \|H_\sigma(a)\| \leq 2M, \, \forall a\in\mathbb{R}^m,
\end{equation}
where $K, M\in\mathbb{R}^+$, and $H_\sigma$ is the Hessian of the scalar field $\sigma$, i.e., $H_\sigma(a)\in\mathbb{R}^{m\times m}$; thus, from the Taylor series of $\sigma$ at $a\in\mathbb{R}^m$ \cite[Theorem 5.15]{WalterRudin1976Principles}, it follows that
\begin{equation}
\label{eq: stay}
| \sigma(a) - \sigma(b) - \nabla\sigma(a)^T(a-b) | \leq M \|a -b\|^2, \, \forall b\in\mathbb{R}^m.
\end{equation}

Now, we are ready to define the source-seeking problem.
\begin{problem}[source seeking]
\label{prob: ss}
Given an unknown signal $\sigma$ and a constant $\epsilon \in \mathbb{R}^+$, find control actions for (\ref{eq: sid}) such that $\|p_c(t) - p_\sigma\| < \epsilon, \forall t \geq T$ for some finite time $T\in\mathbb{R}^+$, where $p_c$ is the centroid of the swarm and $p_\sigma$ is the source of the signal.
\end{problem}

\section{The ascending direction}
\subsection{Discrete robots}
\label{sec: tools}

The authors in \cite{brinon2015distributed} show that the following expression for a circular formation of $N\geq 3$ robots that are equally \emph{angle-spaced}
\begin{equation}
\hat\nabla\sigma(p_c) := \frac{2}{ND^2}\sum_{i=1}^N \sigma(p_i)(p_i - p_c),
\label{eq: lara}
\end{equation}
is an estimation of the gradient at the center $p_c$ of the circumference with radius $D$. However, it is unclear how well (\ref{eq: lara}) approximates the gradient for any generic deployment other than the circle; for example, what is its sensitivity to misplaced robots? In fact, without apparently major changes, the expression (\ref{eq: lara}) can be extended significantly by admitting any generic deployment $x$, i.e., we will prove that the vector
\begin{align}
L_\sigma(p_c,x) &:= \frac{1}{ND^2}\sum_{i=1}^N \sigma(p_i)(p_i - p_c) \nonumber \\ &= \frac{1}{ND^2}\sum_{i=1}^N \sigma(p_c + x_i)x_i, \label{eq: Ls}
\end{align}
is an ascending direction at $p_c$ for some mild conditions, where now $D = \operatorname{max}_{1\leq i\leq N}{\|x_i\|}$. We note that $L_\sigma$ is calculated by the computing unit. We remark that (\ref{eq: Ls}), unlike (\ref{eq: lara}), is a function of a generic deployment $x$, and that the ascending direction is not necessarily parallel to the gradient. We will see that such a property is an advantage to maneuver the swarm by just morphing $x$, or to guarantee an ascending direction in the case of misplaced robots.

Let us approximate $\sigma(p_c + x_i)$ in (\ref{eq: Ls}) to a two-term Taylor series for \emph{small} $\|x_{i}\|$, i.e., $\sigma(p_c + x_i) \approx \sigma(p_c) + \nabla\sigma(p_c)^Tx_i$, so that $L_\sigma(p_c,x) \approx L_\sigma^0(p_c,x) + L_\sigma^1(p_c,x)$, where $L^0_\sigma(p_c,x) = \frac{1}{ND^2}\sum_{i=1}^N \sigma(p_c) x_i = 0$ since $\sum_{i=1}^N x_i = 0$, and
\begin{equation}
L^1_\sigma(p_c,x) = \frac{1}{ND^2}\sum_{i=1}^N \left(\nabla\sigma(p_c)^Tx_i\right)x_i. \label{eq: L1.1}
\end{equation}

Although the important property of $L^1_\sigma$ is its direction, the factor $\frac{1}{ND^2}$ makes $L^1_\sigma$ have the same physical units as the gradient. The direction of the vector $L^1_\sigma$ is interesting because, if $x$ is non-degenerate, then it is an \emph{always-ascending} direction, in the sense that no more conditions are necessary, towards the source $p_\sigma$ at the centroid $p_c$. 

\begin{lemma}
\label{le: l1}
$L^1_\sigma(p_c,x)$ is an always-ascending direction at $p_c$ towards the maximum $p_\sigma$ of the scalar field $\sigma$ if the deployment $x$ is non-degenerate and $p_c \ne p_\sigma$.
\end{lemma}
\begin{proof}
$L^1_\sigma$ is an always-ascending direction if and only if satisfies $\nabla\sigma(p_c)^TL^1_\sigma(p_c, x) > 0$, which is true since 
\begin{equation}
\nabla\sigma(p_c)^TL^1_\sigma(p_c, x) = \frac{1}{ND^2}\sum_{i=1}^N |\nabla\sigma(p_c)^Tx_i|^2 \nonumber
\end{equation}
is positive if the deployment $x$ spans $\mathbb{R}^m$ and $p_c \ne p_\sigma$.
\end{proof}

Lemma \ref{le: l1} motivates us to analyze how fast $L^1_\sigma$ (\ref{eq: L1.1}) diverges from $L_\sigma$ (\ref{eq: Ls}), the actual computed direction, concerning $\|x\|$. 
\begin{lemma}
\label{lem: ll1}
For a signal $\sigma$, a conservative divergence between $L^1_\sigma(p_c,x)$ and $L_\sigma(p_c,x)$ depends linearly on $D$, i.e.,
\begin{equation}
\|L_\sigma(p_c,x) - L^1_\sigma(p_c,x)\| \leq MD,\nonumber
\end{equation}
where $M$ is the upper bound in \eqref{eq_grad_hess_bound}.
\end{lemma}
\begin{proof}
From (\ref{eq: stay}), (\ref{eq: Ls}) and (\ref{eq: L1.1}), then it follows that
\begin{equation}
\scalemath{0.8}{\|L_\sigma - L^1_\sigma\| = \frac{1}{ND^2} \Bigg\|\sum_{i=1}^N \Big(\sigma(p_c + x_i) - \sigma(p_c) - \nabla\sigma(p_c)^Tx_i\Big)x_i\Bigg\|} \nonumber
\end{equation}
$\scalemath{0.8}{\leq \frac{1}{ND^2}\sum_{i=1}^N M\|x_i\|^3 \leq MD.}$
\end{proof}
\begin{remark}
We now have a clear strategy: how to design deployments $x$ such that $L_\sigma^1$ is (almost-)parallel to the gradient $\nabla\sigma$; hence, the (potential) deviation of $L_\sigma$ from $L_\sigma^1$ is still admissible to guarantee that $L_\sigma$ is an ascending direction.
\end{remark}
Indeed, if $D$ is small enough, then it is certain that $L_\sigma$ is an ascending direction like $L_\sigma^1$. However, what is a \emph{small} $D$ for generic signals and deployments? Let us define $E := L_\sigma - L_\sigma^1$. Then we have $\nabla \sigma(p_c)^T L_\sigma = \nabla \sigma(p_c)^T (L_\sigma^1 + E)$, and let us consider a compact set $\mathcal{S} \subset \mathbb{R}^m$ with $p_\sigma \notin \mathcal{S}$. By the definition of $\sigma$ and its bounded gradient, we know that $\nabla \sigma(p_c)^T L_\sigma^1$ has a minimum $F_\mathcal{S}(x)$ that depends on the chosen deployment $x$ and it is positive if $x$ is non-degenerate. Therefore, we have $\nabla \sigma(p_c)^T L_\sigma \geq F_\mathcal{S}(x) - K_\mathcal{S}^{\text{max}} M_\mathcal{S}D$, where $K_\mathcal{S}^{\text{max}}$ and $M_\mathcal{S}$ are the maximum norms of the signal's gradient and Hessian in $\mathcal{S}$, respectively; thus, if
\begin{equation}
\label{eq: xiD}
F_\mathcal{S}(x) - K_\mathcal{S}^{\text{max}} M_\mathcal{S} D > 0,
\end{equation}
then $L_\sigma$ is an ascending direction in $\mathcal{S}$. Finding the minimum $F_\mathcal{S}(x)$ numerically can be arduous, so we can conservatively bound it using the following result.

\begin{lemma}
\label{lem: gradD}
If the deployment $x$ is non-degenerate, then there exists $C(x) > 0$ such that
\begin{equation}
\scalemath{1}{\frac{1}{C(x)}\|\nabla\sigma(p_c)\|^2 \leq L^1_\sigma(p_c, x)^T\nabla\sigma(p_c) \leq C(x) \|\nabla\sigma(p_c)\|^2.} \nonumber
\end{equation}
\end{lemma}
\begin{proof}
First, the trivial case $\nabla\sigma(p_c) = 0$ satisfies the claim. In any other case, we know from Lemma \ref{le: l1} that $\nabla\sigma(p_c)^TL^1_\sigma(p_c, x) = \frac{1}{ND^2}\sum_{i=1}^N |\nabla\sigma(p_c)^Tx_i|^2 = \frac{1}{ND^2} \nabla\sigma(p_c)^T P(x) \nabla\sigma(p_c) > 0$ for the positive definite matrix $P(x)= \sum_{i=1}^N x_ix_i^T$ since $x$ is non-degenerate. Therefore,
\begin{equation}
\frac{\lambda_{\text{min}}\{P(x)\}}{ND^2}  \| \nabla \sigma \|^2 \le \nabla\sigma^T L^1_\sigma  \le \frac{\lambda_{\text{max}}\{P(x)\}}{ND^2}  \| \nabla \sigma \|^2, \nonumber
\end{equation}
where $\lambda_{\text{\{min,max\}}}\{P(x)\}$ are the minimum and maximum eigenvalues of $P(x)$, respectively. We choose $C(x) = \max\{ \frac{\lambda_{\text{max}}\{P(x)\}}{ND^2}, \frac{ND^2}{\lambda_{\text{min}}\{P(x)\}} \}$.
\end{proof}
Later, we will see that $\lambda_{\text{min}}\{P(x)\} = \lambda_{\text{max}}\{P(x)\}$, where $P(x)= \sum_{i=1}^N x_ix_i^T$, for regular polygons/polyhedron deployments, similarly as the equal eigenvalues of a covariance matrix for a uniform deployment in a circle. Now we are ready to give an easier-to-check condition than (\ref{eq: xiD}).
\begin{prop}
\label{pro: S}
Let $\mathcal{S}$ be a compact set with $p_\sigma \notin \mathcal{S}$. Then if 
\begin{equation}
\frac{\lambda_{\text{min}}\{P(x)\}}{ND^2}
 K^{\text{min}}_\mathcal{S} - M_\mathcal{S}D > 0, \nonumber
\end{equation}
for $\forall p_i \in \mathcal{S}$, where $K^{\text{min}}_\mathcal{S}$ is the minimum norm of the gradient in the compact set $\mathcal{S}$, $P(x) = \sum_{i=1}^{N} x_i x_i^T$, then (\ref{eq: Ls}) is an ascending direction at $p_c \in \mathcal{S}$.
\end{prop}
\begin{proof}
Since Lemma \ref{lem: gradD} lower bounds $F_\mathcal{S}(x)$ in (\ref{eq: xiD}) we have 
\begin{align}
\nabla\sigma(p_c)^TL_\sigma &= \nabla\sigma(p_c)^TL^1_\sigma + \nabla\sigma(p_c)^TE \nonumber \\
&= \frac{1}{ND^2}\nabla\sigma(p_c)^TP(x)\nabla\sigma(p_c) + \nabla\sigma(p_c)^TE \nonumber \\
&\geq \frac{\lambda_{\text{min}}\{P(x)\}}{ND^2} \|\nabla\sigma(p_c)\|^2 + \nabla\sigma(p_c)^TE,
\nonumber
\end{align}
and to asses that $L_\sigma(p_c)$ with $p_c\in\mathcal{S}$ is an ascending direction is enough to check $\frac{\lambda_{\text{min}}\{P(x)\}}{ND^2} \|\nabla\sigma(p_c)\|^2 - \|\nabla\sigma(p_c)\|M_\mathcal{S}D > 0$ or $\frac{\lambda_{\text{min}}\{P(x)\}}{ND^2}
 K^{\text{min}}_\mathcal{S} - M_\mathcal{S}D > 0$.
\end{proof}

We note that the covariance matrix $\frac{1}{N}\sum_{i=1}^Nx_ix_i^T$ is \emph{normalized} by $D^2$ for checking the condition in Proposition \ref{pro: S}, making the effective dependency on $D$ linear and matching the result of Lemma \ref{lem: ll1}.

While the signal $\sigma$ is unknown, we can always engineer $D$ depending on the expected scenarios. For example, in the scenario of contaminant leakage, scientists can provide the expected values of $M_\mathcal{S}, K_\mathcal{S}^{\text{max}}$, and $K_\mathcal{S}^{\text{min}}$ in the \emph{patrolling area} $\mathcal{S}$ for the minimum/maximum contamination thresholds where the robot team needs to react reliably. We also note that designing $L^1_\sigma$ to be parallel to the gradient $\nabla\sigma$ makes the scalar product $\nabla\sigma^TL_\sigma > 0$ more robust concerning $D$ since it makes admissible a larger deviation of $L_\sigma$ from $L^1_\sigma$, e.g., misplaced robots from a reference deployment $x$. 


The analysis is organized as follows: first, we examine the four-robot rectangular formation. Second, we explore regular polygons and affine transformations, i.e., the morphing of $x$, to assess the sensitivity of the ascending direction. Finally, we investigate deploying an infinite number of robots with a density distribution within a given shape.

\begin{figure}
\centering
\includegraphics[trim={1.5cm 4.4cm 0 5.657cm}, clip, width=0.85\columnwidth]{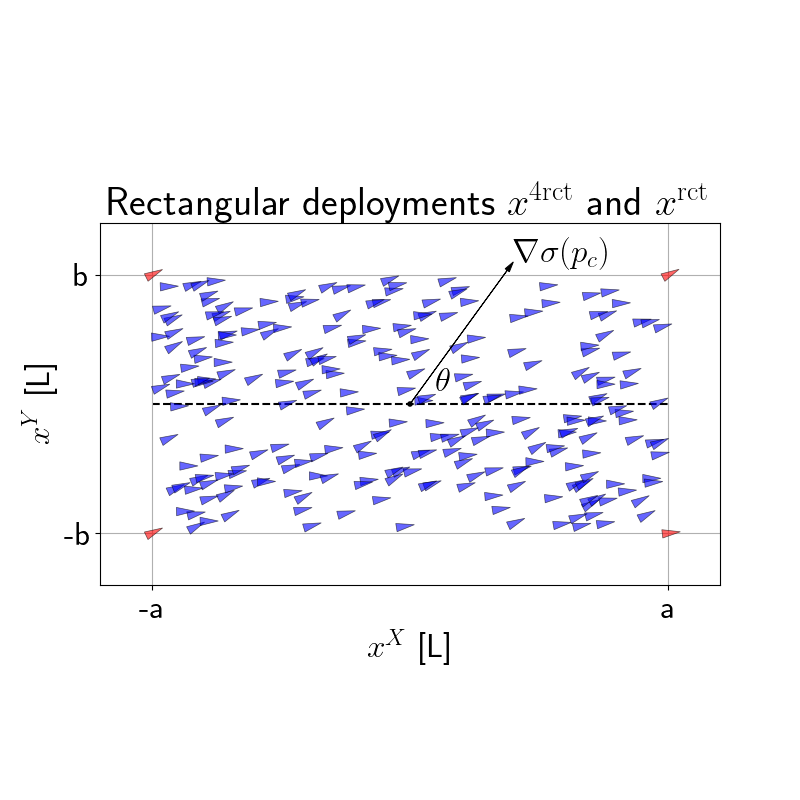}
\caption{The 4-robot rectangular deployment $x^\text{4rct}$ are at the corners of the rectangle in red color. The $250$ robots in blue color belong to the deployment $x^\text{rct}$ and they are spread uniformly within the rectangle. The gradient $\nabla\sigma(p_c)$ is arbitrary and forms an angle $\theta$ with the horizontal axis of the deployment.} 
\label{fig: 4rect}
\end{figure}

Consider four robots at the corners of a rectangle as deployment $x^{4\text{rct}}$, and assume that the long side of the rectangle is parallel to the horizontal axis so that the gradient $\nabla\sigma(p_c) = \|\nabla\sigma(p_c)\| \left[\begin{smallmatrix}\cos(\theta) & \sin(\theta) \end{smallmatrix}\right]^T$ has an arbitrary norm and angle as depicted in Figure \ref{fig: 4rect}. Then, the always-ascending direction $L^1_{\sigma}$ can be written as
\begin{equation}
\scalemath{0.8}{
L^1_\sigma(p_c, x^{4\text{rct}}) = \frac{\|\nabla\sigma(p_c)\|}{4D^2} \sum_{i=1}^4 \left(\begin{bmatrix}\cos(\theta) & \sin(\theta) \end{bmatrix} \begin{bmatrix}x_i^X \\ x_i^Y \end{bmatrix}\right)\begin{bmatrix}x_i^X \\ x_i^Y \end{bmatrix},}
\label{eq: L1theta}
\end{equation}
where the superscripts $X$ and $Y$ denote the horizontal and vertical coordinates of $x_i$ respectively. Therefore, for the deployment $x^{4\text{rct}}$ we have that
\begin{align}
&\scalemath{0.8}{L^1_\sigma(p_c, x^{4\text{rct}}) = \frac{\|\nabla\sigma(p_c)\|}{4(a^2+b^2)}\bigg(\big(a\cos(\theta)+b\sin(\theta)\big)\begin{bmatrix}a\\b \end{bmatrix} +} \nonumber \\ 
&\scalemath{0.8}{+\big(-a\cos(\theta)+b\sin(\theta)\big)\begin{bmatrix}-a\\b \end{bmatrix}
+ \big(-a\cos(\theta)-b\sin(\theta)\big)\begin{bmatrix}-a\\-b \end{bmatrix} + \nonumber} \\
& \scalemath{0.8}{+\big(a\cos(\theta)-b\sin(\theta)\big)\begin{bmatrix}a\\-b \end{bmatrix} \bigg) = \frac{\|\nabla\sigma(p_c)\|}{(a^2 + b^2)}\begin{bmatrix}a^2\cos(\theta)\\b^2 \sin(\theta)\end{bmatrix},} \label{eq: l1sq4}
\end{align}
	where we can observe that if $a$ or $b$ equals zero, then only the projection of $\nabla\sigma(p_c)$ onto the line described by the \emph{degenerate} $x$ in $L^1_\sigma(p_c, x^{4\text{rct}})$ is observed. We can also see that for the square $a=b$, we have that $L^1_\sigma(p_c, x^{4\text{rct}}) = \frac{1}{2}\nabla\sigma(p_c)$. In general, $L^1_\sigma(p_c, x^{N\text{poly}}) \propto \left[\begin{smallmatrix}\cos\theta & \sin\theta \end{smallmatrix}\right]^T$, i.e., it is parallel to $\nabla\sigma(p_c)$, for any deployment $x^{N\text{poly}}$ forming a regular polygon or polyhedron as it is shown in the following technical result.
	\begin{lemma}
	\label{le: poly}
	Consider the deployment $x^{N\text{poly}}$ forming a regular polygon or polyhedron, and $b\in\mathbb{R}^m$, then
	\begin{equation}
	\sum_{i=1}^N \left(b^Tx_i\right)x_i \propto b. \nonumber 
	\end{equation}
	\end{lemma}
	\begin{proof}
	First, we have that $\sum_{i=1}^N \left(b^T x_i\right)x_i = \sum_{i=1}^N(x_ix_i^T)b$; then, in order to be proportional to $b$, we need all the diagonal elements of the positive semi-definite matrix $P = \sum_{i=1}^N(x_ix_i^T)$ be equal, and the non-diagonal to be zero. Since all the considered deployments lie on the $(m-1)$-sphere, all the $\|x_i\|$ are equal. In addition, all the dihedral angles of the considered deployments are equal; then we can find XYZ axes where the set of all projections of the vertices $x_i$ on the planes XY, XZ, and YZ, exhibit an even symmetry; hence, $\sum_i^N (x_i^X)^2 = \sum_i^N (x_i^Y)^2 = \sum_i^N (x_i^Z)^2$. Finally, a regular polygon has reflection symmetry for one axis of the same XY; hence, $\sum_i^N (x_i^X)(x_i^Y) = 0$. This is also true for a regular polyhedron for the same planes XY, XZ, and YZ.
	\end{proof}

	A similar result has been shown in \cite{brinon2015distributed, brinon2019multirobot} with their circular and spherical formations since regular polygons and polyhedra are inscribed in the circle and sphere. We now analyze the sensitivity of $L^1_\sigma(p_c, x^{N\text{poly}})$ when the deployment is under an affine transformation, e.g., scaling, rotation, and shearing. The formal affine transformation is given by $(I_N \otimes A)x^{N\text{poly}}$ with $A\in\mathbb{R}^{m\times m}$ where $A$ admits the Singular Value Decomposition (SVD) $A = U\Sigma V^T$ describing a rotation, scaling, and another rotation. 

	\begin{prop}
	\label{pro: usu}
	Consider the deployment $x^{N\text{poly}}$ forming a regular polygon or polyhedron, and the SVD $A = U\Sigma V^T$, then
	\begin{equation}
	L^1_\sigma(p_c,(I_N \otimes A) x^{N\text{poly}}) \propto U\Sigma^2U^Tr, \nonumber
	\end{equation}
	where $r\in\mathbb{R}^m$ is the unitary vector marking the direction of the gradient $\nabla\sigma(p_c)$.
	\end{prop}
	\begin{proof}
	Considering the change of coordinates $\tilde r = V\Sigma U^Tr$
	\begin{align}
	&\scalemath{0.8}{L^1_\sigma(p_c,(I_N \otimes (U\Sigma V^T)) x^{N\text{poly}}) \propto \sum_{i=1}^N \Big(r^T U\Sigma V^T x_i\Big)U\Sigma V^T x_i} \nonumber \\
	&\scalemath{0.8}{\propto U\Sigma V^T\sum_{i=1}^N \Big(\tilde r^T x_i\Big) x_i \propto U\Sigma V^T \tilde r \nonumber = U\Sigma V^T V \Sigma U^Tr=  U\Sigma^2 U^Tr,}
	\end{align}
where we have applied Lemma \ref{le: poly}.
\end{proof}

Note that $W = U\Sigma^2U^T$ is the unitary decomposition of a positive definite matrix, and $V$ is irrelevant (for example, $\theta$ is arbitrary in 2D in Figure \ref{fig: 4rect}), as shown in Lemma \ref{le: poly}. Stretching the deployment $x^{N\text{poly}}$ results in $L_\sigma^1$ following the direction ${^\xi\Sigma^2}\,\,{^\xi\nabla\sigma(p_c)}$, where $\xi$ denotes the \emph{stretching axes}. Such a morphing assists with maneuvering the swarm while it gets closer to the source.

\subsection{Continuous distributions of robots}
Looking at the definition of $L_\sigma$ in (\ref{eq: Ls}), it is clear that we can apply the superposition of many, potentially infinite, deployments. Indeed, we can step to \emph{continuous} deployments considering $N\to\infty$ within a specific area or volume. In such a case, the finite sum in (\ref{eq: Ls}) becomes an integral with a robot density function following the approximation of a definite integral with Riemann sums, and $ND^2$ in (\ref{eq: Ls}) becomes $A = \iiint_{\mathcal{A}} \rho(X,Y,Z) \mathrm{d}X\mathrm{d}Y\mathrm{d}Z$, where $\mathcal{A}$ is the corresponding surface/volume or \emph{perimeter} of $x$, and $\rho: \mathcal{A} \to \mathbb{R}^+$ is a normalized (robot) density function, e.g., equal everywhere for uniform distributions as it is shown in Figure \ref{fig: 4rect} with $x^{rct}$. 

For the sake of conciseness, let us focus on 2D for the following results, and for the sake of clarity with the notation, we will denote $x_i^X$ and $x_i^Y$ as simply $X$ and $Y$. Accordingly, for the case of a robot swarm described by the density function $\rho(X,Y)$ of robots within a generic shape/surface $\mathcal{A}$, the always-ascending direction can be calculated as
\begin{align} 
    \scalemath{0.8}{L_\sigma^1(p_c, x)}  & \scalemath{0.8}{= 
    \frac{\|\nabla\sigma(p_c)\|}{A}  \iint_\mathcal{A} \rho(X,Y) \left(
    \begin{bmatrix}\cos(\theta) & \sin(\theta) \end{bmatrix} 
    \begin{bmatrix}X \\ Y \end{bmatrix}
    \right)
    \begin{bmatrix}X \\ Y \end{bmatrix} 
    \mathrm{d}X \mathrm{d}Y}
    \nonumber \\
    & \scalemath{0.8}{= \frac{\|\nabla\sigma(p_c)\|}{A} \iint_\mathcal{A}  \rho(X,Y)  \begin{bmatrix}
        X^2 \cos(\theta) + XY\, \sin(\theta) \\
        Y^2 \sin(\theta) + XY\, \cos(\theta) \\
    \end{bmatrix}
    \mathrm{d}X \mathrm{d}Y.} \nonumber
\end{align}
Similar to the discrete case, e.g., see (\ref{eq: l1sq4}), it is enough to have $\iint_\mathcal{A} \rho(X,Y)XY  \mathrm{d}X\mathrm{d}Y$ and $\iint_\mathcal{A} \rho(X,Y)(X^2 - Y^2) \mathrm{d}X\mathrm{d}Y$ equal to zero to have $L^1_\sigma$ parallel to the gradient.

\begin{prop}
\label{pro: U}
Consider a signal $\sigma$, a continuous deployment $x$ with robot density function $\rho(X,Y)$ within a surface $\mathcal{A}$, and a Cartesian coordinate system $(X-Y)$ with origin at the centroid of the deployment $p_c$. The direction of $L_\sigma^1(p_c, x)$ is parallel to the gradient $\nabla\sigma(p_c)$ if $\rho(X,Y)$ and $\mathcal{A}$ hold the following symmetries:
\begin{enumerate}
\item[S0)] The robot density function $\rho(X,Y)$ has reflection symmetry (even function) concerning at least one of the axes $(X-Y)$, e.g., $\rho(X,Y) = \rho(-X,Y)$.
\item[S1)] The surface $\mathcal{A}$ has reflection symmetry concerning the same axes as in S0.
\item[S2)] For each quadrant of $(X-Y)$, the robot density function $\rho(X,Y)$ has reflection symmetry concerning the bisector of the quadrant.
\item[S3)] The surface $\mathcal{A}$ has reflection symmetry concerning the same axes as in S2.
\end{enumerate}
\end{prop}
\begin{proof}

Without loss of generality, we assume that the reflection symmetry is on the $Y$ axis, and we assume symmetric integration limits for the horizontal axis as in Figure \ref{fig: obs_xxyy}(a). First, if $\alpha(X)$ and $\beta(X)$ are both even functions as illustrated in Figure \ref{fig: obs_xxyy}(a), we are going to show that
\begin{align}
&\scalemath{0.8}{\iint_{\mathcal{A}} \rho(X,Y) XY \; \mathrm{d}X \mathrm{d}Y =} \nonumber 
\\&\scalemath{0.8}{\quad
\int_{-t_\beta}^{t_\beta}\int_{0}^{\beta(X)} \rho(X,Y) XY \, \mathrm{d}X  \mathrm{d}Y
+
\int_{-t_\alpha}^{t_\alpha}\int_{\alpha(X)}^{0} \rho(X,Y) XY \, \mathrm{d}X \mathrm{d}Y} \nonumber
\\&\scalemath{0.8}{\quad= 
\int_{-t_\beta}^{t_\beta} X F(X,\beta(X)) \, \mathrm{d}X  
-
\int_{-t_\alpha}^{t_\alpha} X F(X,\alpha(X)) \, \mathrm{d}X  = 0},
\label{eq: int0}
\end{align}
 where $F(X,Y) = \int \rho(X,Y) Y \; \mathrm{d}Y$. Given an arbitrary even function $f(X)$ and knowing that $\rho(X,Y) = \rho(-X,Y)$, we have that $F(X,Y) = F(-X,Y)$ and $F(X,Y) = F(X,f(X)) = F(-X,f(-X)) = F(X,-Y)$; therefore, both integrals in (\ref{eq: int0}) are zero considering symmetries S\{0,1\}. 

Next, we will divide $\mathcal{A}$ into four parts for each quadrant, as it is shown in Figure \ref{fig: obs_xxyy}(b), so that

\begin{equation}
\scalemath{0.8}{
\iint_{\mathcal{A}} \rho(X,Y) (X^2-Y^2) \; \mathrm{d}X \mathrm{d}Y 
=
\sum_{i=1}^4
\iint_{\mathcal{A}_i} \rho(X,Y) (X^2-Y^2) \; \mathrm{d}X \mathrm{d}Y
}, \nonumber
\end{equation}
where $\mathcal{A}_i$ is the area of the quadrant $i \in \{1,2,3,4\}$. From the symmetry exhibited by $X^2-Y^2$ we propose the following change of variables $g(\epsilon,\psi) = (\psi + \epsilon, \psi - \epsilon)/\sqrt{2}$, which is equivalent to a rotation of $-\frac{\pi}{4}$ radians for the $(X-Y)$ axes. Since $\int_{A} f(x,y) = \int_{B} (f \circ g) |J_g|$, where $|J_g| = |\begin{bmatrix}\nabla g_1 & \nabla g_2 \end{bmatrix}^T| = \sqrt{2}$, we have

\begin{align*}
&\scalemath{0.8}{\iint_{\mathcal{A}_i} \rho(X,Y) (X^2-Y^2) \; \mathrm{d}X \mathrm{d}Y
= 2 \sqrt{2} \iint_{\mathcal{B}_i} \rho(\epsilon,\psi) \epsilon \psi \; d\epsilon d\psi} \\ 
&\scalemath{0.8}{= 2 \sqrt{2} \int_{-t_i}^{t_i}\int_{|\epsilon|}^{R_i(\epsilon)} \rho(\epsilon,\psi) \epsilon \psi \; d\epsilon d\psi
= 
2 \sqrt{2} \int_{-t_i}^{t_i} \epsilon \left [F(\epsilon,R_i(\epsilon)) - F(\epsilon,|\epsilon|) \right ] d\epsilon},
\end{align*}
which is zero if $R_i(\epsilon) = R_i(-\epsilon)$ for all quadrants since $|\epsilon|$ is an even function. Such condition is satisfied if $\alpha(X)$ and $\beta(X)$ have reflection symmetry concerning the bisector of the two lower and upper quadrants respectively as illustrated in Figure \ref{fig: obs_xxyy}(b); thus, symmetries S\{2,3\} are checked.\end{proof}

\begin{figure}
    \centering
    \includegraphics[trim={3cm 4cm 5.0cm 3.5cm}, clip, width=0.85\columnwidth]{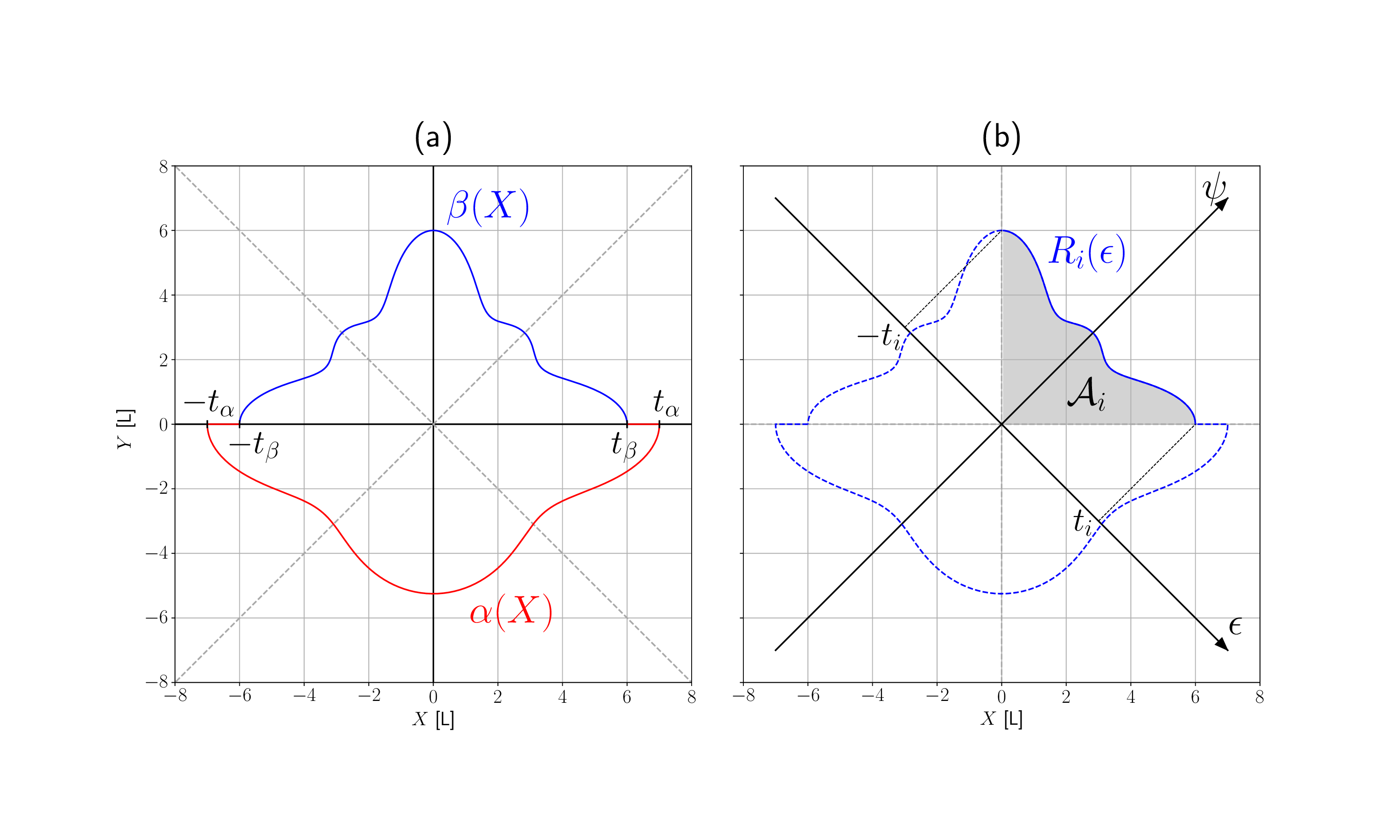}
    \caption{Illustration of the symmetries S\{1,3\} for the surface $\mathcal{A}$ on the left and right respectively, to design $L^1_\sigma(p_c, x)$ parallel to the gradient $\nabla\sigma(p_c)$.}
    \label{fig: obs_xxyy}
\end{figure}

\begin{figure*}[t!]
\centering
\includegraphics[trim={1cm 1.4cm 0cm 1.1cm}, clip, width=1.6\columnwidth]{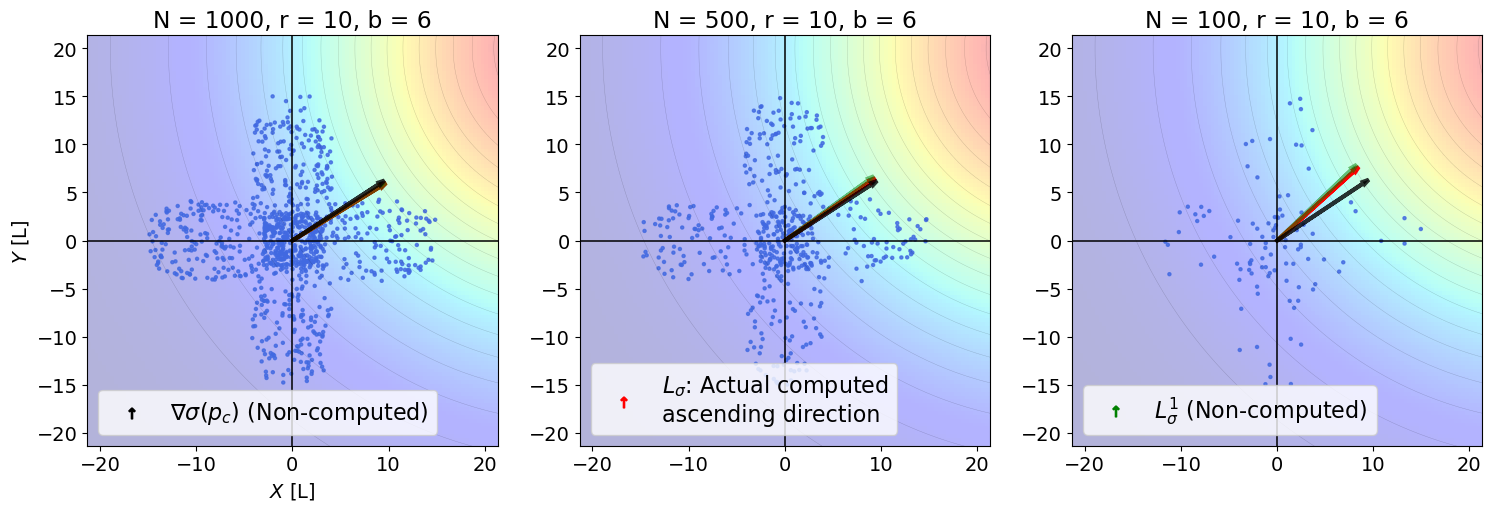}
\caption{On the left, a swarm of $1000$ robots satisfying the symmetries $S\{0,1,2,3\}$ from Proposition \ref{pro: U}. For the three figures, the black arrows are the gradient $\nabla\sigma$, the green arrows are $L^1_\sigma$, and the red arrows are the direction $L_\sigma$ computed by the robots. Arrows are normalized for representation purposes as we focus only on their direction. As we get far from \emph{the continuum}, with $500$ and $100$ robots on the middle and right figures, there is no guarantee that $L^1_\sigma$ is parallel to the gradient $\nabla\sigma$. However, in this example, $L_\sigma$ has almost the same direction as $L_\sigma^1$ despite a \emph{big} $D$ concerning the level curves.}
\label{fig: prop5}
\end{figure*}

\begin{figure*}[t!]
\centering
\includegraphics[trim={1cm 1.4cm 0cm 1.05cm}, clip, width=1.6\columnwidth]{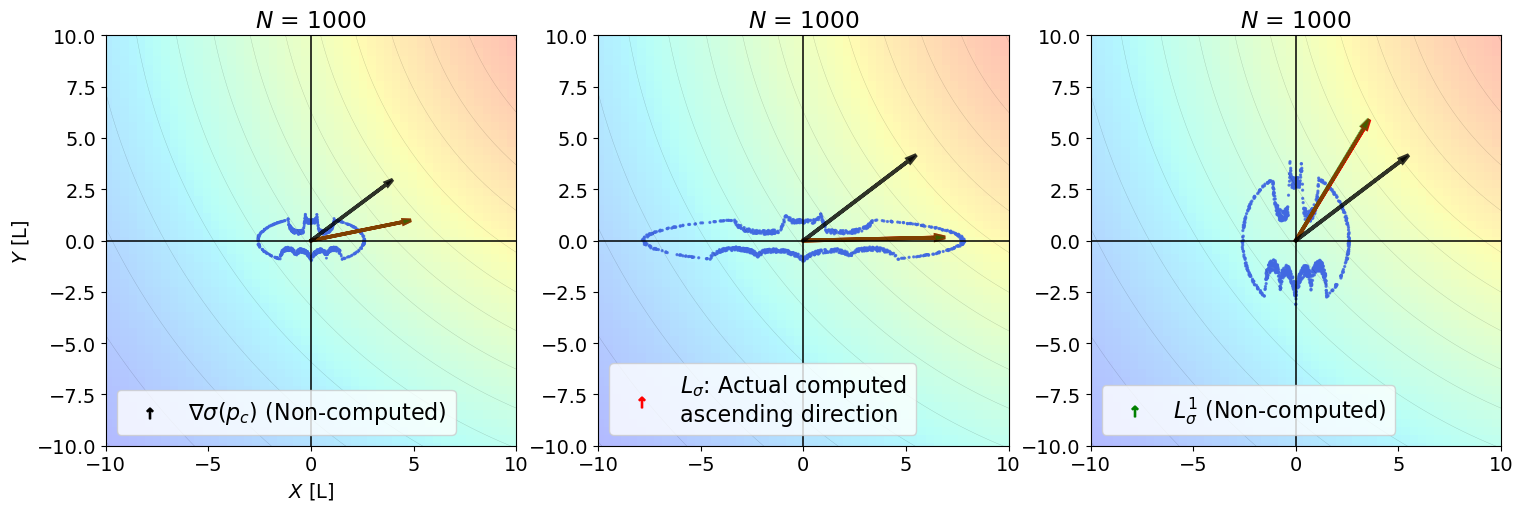}
\caption{On the left, a swarm of $1000$ robots satisfying only the symmetries $S\{0,1\}$ from Proposition \ref{pro: U}. The arrows have the same meaning as in Figure \ref{fig: prop5}. The modification of the variances of $\rho(X,Y)$, e.g., by stretching the swarm shape, stretches $L^1_\sigma$ and $L_\sigma$ accordingly. In this example, the direction of $L_\sigma$ almost does not diverge from $L_\sigma^1$ despite a \emph{big} $D$ concerning the level curves.}
\label{fig: Batman}
\end{figure*}

Indeed, not only \emph{uniform} deployments $x^{\text{Upoly}}$ describing an area/volume of a regular polygon/polyhedron fit into Proposition \ref{pro: U}, but a richer collection of deployments like the ones depicted in Figures \ref{fig: obs_xxyy} and \ref{fig: prop5}. In fact, if only $S\{0,1\}$ are satisfied, concerning $\rho(X,Y)$, the variances $\text{VAR}_{\mathcal{A}}[X] \neq \text{VAR}_{\mathcal{A}}[Y]$ in general, so we have that
\begin{equation} 
\scalemath{0.8}{
    L_\sigma^1(p_c, x)  = 
     \frac{\|\nabla\sigma(p_c)\|}{A} \begin{bmatrix}
        \text{VAR}_{\mathcal{A}}[X] \cos(\theta)\\
        \text{VAR}_{\mathcal{A}}[Y] \sin(\theta)\\
    \end{bmatrix}}, \nonumber
\end{equation}
where the relation between variances \emph{maneuvers} the swarm while getting closer to the source, as shown in Figure \ref{fig: Batman}.

\begin{figure*}[t!]
\centering
\includegraphics[trim={0cm 0cm 0cm 0cm}, clip, width=2\columnwidth]{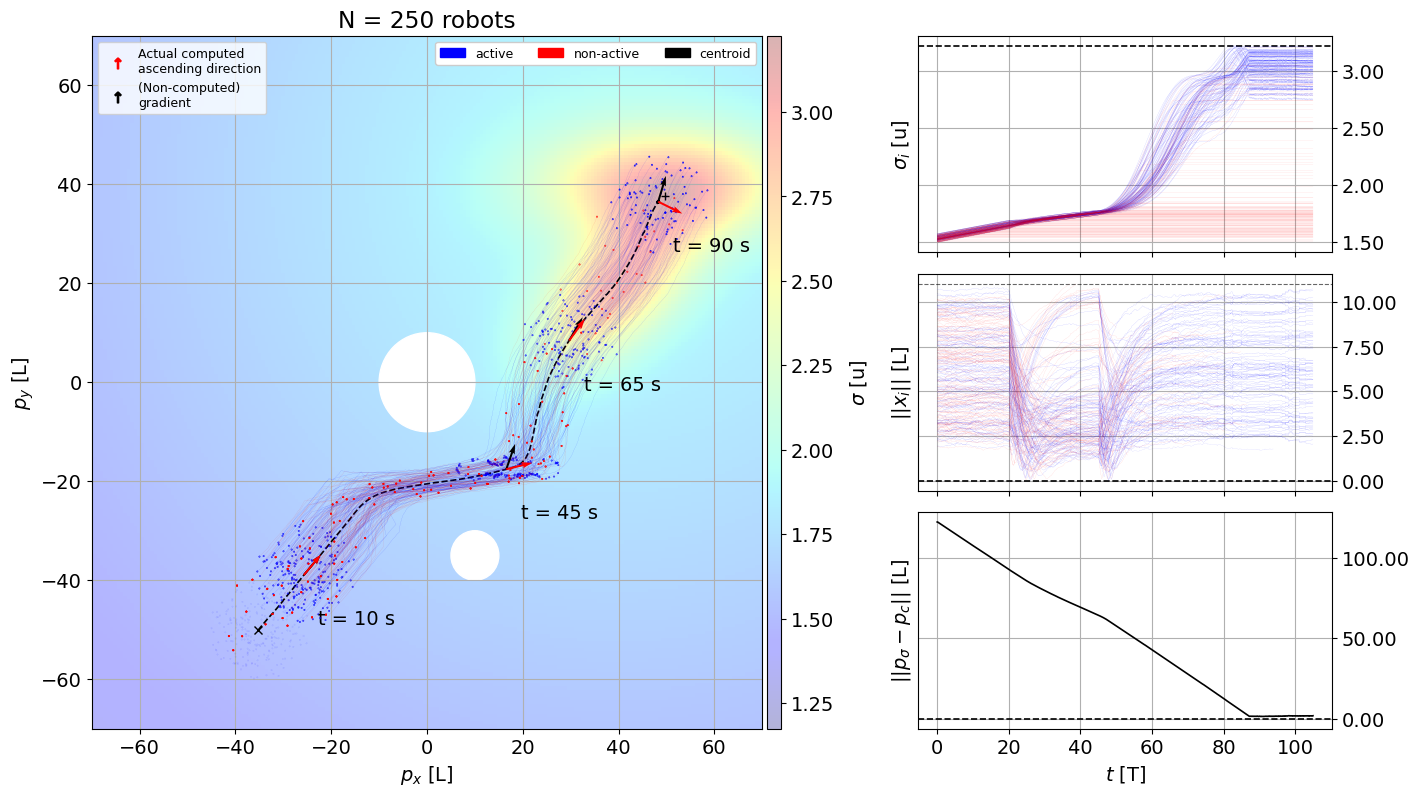}
\caption{A robot swarm of $250$ individuals (blue dots) faces faulty actuators with noisy directions, and individuals stop working randomly (turned into red dots). The swarm can maneuver by morphing its shape, e.g., to pass between the two white circles, and reaches the signal source eventually, where $170$ robots were dead during the mission; nonetheless, the rest made it to the source, demonstrating the resiliency of the swarm. On the right side, from top to down, are the readings of the signal by the robots, the distances from the robots to the swarm centroid, and the distance of the swarm centroid to the source.
}
\label{fig: sim}
\end{figure*}

\section{source seeking with single-integrator robots.}
\label{sec: sianalysis}

Focusing on robots with the single-integrator dynamics (\ref{eq: sid}), the following theorem claims that $L_\sigma(p_c(t), x)$ can guide the centroid of the robot swarm effectively to the signal source.

\begin{theorem}
\label{th: si}
Consider the signal $\sigma$ and a swarm of $N$ robots with the stacked positions $p(t)\in\mathbb{R}^{mN}$, where $m=\{2,3\}$. Assume that the initial deployment $x^*$ is non-degenerate. Then
\begin{equation}
\begin{cases}
\dot p_i(t) &= \frac{L_\sigma(p_c(t), x(t))}{\|L_\sigma(p_c(t), x(t))\|} \\
p_i(0) &= p_c + x^*_i, \quad p_c \neq p_\sigma
\end{cases}, \forall i\in\{1,\dots,N\},
	\label{eq: th}
\end{equation}
is a solution to Problem \ref{prob: ss} given the conditions of Proposition \ref{pro: S} for an annulus shell $\mathcal{S}$ around $p_\sigma$.
\end{theorem}

\begin{proof}
We first consider the dynamics of the following system
\begin{equation}
\begin{cases}
\dot p_i(t) &= L^1_\sigma(p_c(t), x(t)) \\
p_i(0) &= p_c + x^*_i, \quad p_c \neq p_\sigma
\end{cases}, \forall i\in\{1,\dots,N\},
\label{eq: th1}
\end{equation}
	and then we continue to analyse the system (\ref{eq: th}) in the statement. Let us consider the function $\sigma(p_c(t))$ with a positive maximum at $p_\sigma$, and without loss of generality, let us set $p_\sigma$ at the origin so that $\sigma(p) \to \operatorname{sup}\{\sigma\}=\sigma(0)$ as $\|p\|\to 0$, and $\sigma(p) \to 0$ as $\|p\|\to\infty$. Because of Definition \ref{signal}, $\sigma(p)$ does not have any local maxima or minima, but only the global maximum $\operatorname{sup}\{\sigma\}=\sigma(0)$. The dynamics of $\sigma(p_c(t))$ under (\ref{eq: th1}) is: 
\begin{align}
\dot\sigma(p_c(t)) = \nabla\sigma(p_c)^T \dot p_c(t) = \nabla\sigma(p_c)^T L^1_\sigma(p_c(t), x(t)) \geq 0, \nonumber
\end{align}
since $x(t)$ is not degenerated for all $t$, then the equality holds if and only if $p_c = p_\sigma$ because of Lemma \ref{le: l1}.
%



Let us consider the Lyapunov function $V = (\sigma(p_c(t)) - \sigma(0))^2/2$, and calculating the time derivative with respect to \eqref{eq: th1}, we have
\begin{align*}
\dot{V}(t) &= (\sigma(p_c(t)) - \sigma(0)) \nabla \sigma^T(p_c(t)) \dot p_c(t) \\ 
&= \underbrace{(\sigma(p_c(t)) - \sigma(0)) \nabla \sigma^T(p_c(t)) L^1_\sigma(p_c(t), x(t))}_{W(t)} \leq 0.
\end{align*}
At this point, we can invoke LaSalle's invariance principle to conclude that $W(t) \to 0$ as $t \to \infty$; or equivalently $\lim_{t\to\infty} p_c(t) = p_\sigma=0$ since we recall that $x(t)$ is not degenerate. Therefore, given an $\epsilon > 0$, there is always a time $T^* > 0$, such that $\|p_c(t) - p_\sigma\| < \epsilon, \forall t > T^*$, as required to solve the source-seeking Problem \ref{prob: ss}. 

	Let us focus now on the system (\ref{eq: th}). We first know that $\frac{L_\sigma(p_c(t), x(t))}{\|L_\sigma(p_c(t), x(t))\|}$ is well defined in $\mathcal{S}$ if $x(t)$ is not degenerate. Let us consider the annulus shell $\mathcal{S}$ centered at $p_\sigma$ and with radii $\epsilon^+ > \epsilon_- > 0$, then, according to Proposition \ref{pro: S}, there exists $D^*(\epsilon^+,\epsilon_-)$ such that if $D \leq D^*$ the trajectory of $p_c(t)$ under the dynamics (\ref{eq: th}) is an ascending one so that we can apply the same arguments as for the system (\ref{eq: th}) and solve the source-seeking Problem \ref{prob: ss} where its $\epsilon = \epsilon_-$ and the initial condition must satisfy $||p_c(0)|| \leq \epsilon^+$. For all $t>T^*$ such that $\|p_c(T^*) - p_\sigma\| < \epsilon_-$, we note that since the robots have single integrator dynamics, if the centroid of the swarm comes back to the compact set $\mathcal{S}$ it will only stay at its closure and leave it again since $L_\sigma$ is an ascending direction at the closure of $\mathcal{S}$. If $L_\sigma$ turns out to be zero outside of $\mathcal{S}$ then the swarm can stop.

\end{proof}

\section{Numerical simulation}
\label{sec: sce}

Figure \ref{fig: sim} shows a $250$ robot swarm seeking the source of the non-convex signal,
\begin{equation*}
\begin{split}
\sigma(p)=& 2 - 0.04\left \Vert (p - p_\sigma) \right \Vert +\exp\left \{-0.9(p-p_\sigma)^TA_1u_X\right \}\\
+&\exp \left \{0.9(p-p_\sigma)^TS^TA_2Su_Y \right \},
\end{split}
\end{equation*}
with
$u_X = \left[\begin{smallmatrix}
1\\0
\end{smallmatrix}\right]$, 
$u_Y = \left[\begin{smallmatrix}
0\\1
\end{smallmatrix}\right]$, 
$p_\sigma = 40(u_X+u_Y)$, $A_1 = \left[\begin{smallmatrix}
\frac{1}{\sqrt{30}} & 0\\1 & 0 
\end{smallmatrix}\right]$, 
$A_2 = \left[\begin{smallmatrix}
1 & 0\\1 & \frac{1}{\sqrt{15}}
\end{smallmatrix}\right]$ and $S = \left[\begin{smallmatrix}
1 & -1\\ 1 & \ \ 1 
\end{smallmatrix}\right]$, 
  by using the direction $L_\sigma$ (red arrow in Figure \ref{fig: sim}) in Theorem \ref{th: si}, i.e., robots move at an unitary speed. The swarm starts as a uniform circular cloud of radius $10$ length units. We show the resilience of the swarm by introducing the following three factors. Firstly, in order to show \emph{misplacements} within the formation, we inject noise to the actuators, i.e., every $0.2$ time units the robots track the direction $L_\sigma$ with a random deviation within $\pm 10$ degrees. Secondly, at time $20$, the swarm maneuvers between two obstacles by morphing into the shape in Figure \ref{fig: Batman} since it stretches $L_\sigma$ horizontally. Thirdly, there is a random probability for individual robots to stop functioning. These \emph{dead} robots are represented in red color, and only $80$ robots made it eventually to the source. At time $90$, the centroid swarm gets so close to the source that $L_\sigma$ becomes unreliable, i.e., the centroid entered the $\epsilon$-ball in Problem \ref{prob: ss}.

\section{Conclusions and future work}
\label{sec: con}
Using a robot swarm that always calculates an ascending direction regardless of the shape of its spatial deployment offsets individual losses and mitigates noise issues for the source-seeking problem, i.e., it adds resiliency. Our next steps include a more distributed computation of the ascending direction $L_\sigma$, including a fully distributed algorithm to estimate the centroid estimation, which is compatible with our source-seeking method, and we will analyze how our algorithm works with more restrictive agents' dynamics such as unicycles traveling at constant speeds.

\bibliographystyle{IEEEtran}
\bibliography{biblio}

\end{document}